\documentclass{article}

\usepackage{times,amsmath,amsfonts,amssymb,url,color,graphicx}

\usepackage{boxedminipage}

\usepackage{subfigure} 

\usepackage{natbib}

\usepackage{algorithm}
\usepackage{algorithmic}

\usepackage{hyperref}

\usepackage[accepted]{icml2016}

\newtheorem{lemma}{Lemma}

\newtheorem{theorem}{Theorem}

\newcommand{\BlackBox}{\rule{1.5ex}{1.5ex}}  
\newenvironment{proof}{\par\noindent{\bf Proof\ }}{\hfill\BlackBox\\[2mm]}

\newcommand{\E}{\mathbb{E}}

\DeclareMathOperator*{\argmin}{argmin} 
 
\newcommand{\reals}{\mathbb{R}}

\newcommand{\secref}[1]{Section~\ref{#1}}

\newcommand{\thmref}[1]{Theorem~\ref{#1}}

\newcommand{\lemref}[1]{Lemma~\ref{#1}}

\newenvironment{myalgo}[1]%
{
\begin{center}
\begin{boxedminipage}{0.8\linewidth}
\begin{center}
\textbf{\texttt{#1}}
\end{center}
\rm
\begin{tabbing}
....\=...\=...\=...\=...\=  \+ \kill
} %
{\end{tabbing} 
\end{boxedminipage} \end{center} 
}

\icmltitlerunning{SDCA without Duality, Regularization, and Individual Convexity}

\begin{document} 

\twocolumn[
\icmltitle{SDCA without Duality, Regularization, and Individual Convexity}

\icmlauthor{Shai Shalev-Shwartz}{shais@cs.huji.ac.il}
\icmladdress{School of Computer Science and Engineering, The Hebrew
  University of Jerusalem, Israel}

\icmlkeywords{}

\vskip 0.3in
]

\begin{abstract}
  Stochastic Dual Coordinate Ascent is a popular method for solving
  regularized loss minimization for the case of convex losses.  We
  describe variants of SDCA that do not require explicit
  regularization and do not rely on duality. We prove linear
  convergence rates even if individual loss functions are non-convex,
  as long as the expected loss is strongly convex.
\end{abstract}

\section{Introduction}

We consider the following loss minimization problem:
\[
\min_{w \in \reals^d} F(w) := \frac{1}{n} \sum_{i=1}^n f_i(w) ~.
\]
An important sub-class of problems is when each $f_i$ can be
written as $f_i(w) = \phi_i(w) + \frac{\lambda}{2} \|w\|^2$, where
$\phi_i$ is $L_i$-smooth and convex. A popular method for
solving this sub-class of problems is Stochastic Dual Coordinate
Ascent (SDCA), and \cite{ShalevZh2013} established the convergence
rate of $\tilde{O}((L_{\max}/\lambda + n)\log(1/\epsilon))$, where
$L_{\max} = \max_i L_i$.

As its name indicates, SDCA is derived by considering a dual
problem. In this paper, we consider the possibility of applying SDCA
for problems in which individual $f_i$ do not necessarily have the
form $\phi_i(w) + \frac{\lambda}{2} \|w\|^2$, and can even be
non-convex (e.g., deep learning optimization problems, or problems
arising in fast calculation of the top singular vectors~\cite{jin2015robust}). In many such
cases, the dual problem is meaningless. Instead of directly using the
dual problem, we describe and analyze a variant of SDCA in which only
gradients of $f_i$ are being used. Following
\cite{johnson2013accelerating}, we show that SDCA is a member of the
Stochastic Gradient Descent (SGD) family of algorithms, that is, its
update is based on an unbiased estimate of the gradient, but unlike
the vanilla SGD, for SDCA the variance of the estimation of the
gradient tends to zero as we converge to a minimum.

Our analysis assumes that $F$ is $\lambda$-strongly convex and each
$f_i$ is $L_i$-smooth. When each $f_i$ is also convex we establish the
convergence rate of $\tilde{O}(\bar{L}/\lambda + n)$, where $\bar{L}$
is the average of $L_i$ and the $\tilde{O}$ notation hides logarithmic
terms, including the factor $\log(1/\epsilon)$. This matches the best
known bound for SVRG given in \cite{xiao2014proximal}. Lower bounds
have been derived in \cite{arjevani2015lower,agarwal2014lower}.
Applying an acceleration technique
(\cite{ShalevZhangAcc2015,lin2015universal}) we obtain the convergence
rate $\tilde{O}(n^{1/2}\,\sqrt{\bar{L}/\lambda} + n)$.  If $f_i$ are
non-convex we first prove that SDCA enjoys the rate
$\tilde{O}(\bar{L}^2/\lambda^2 + n)$. Finally, we show how the
acceleration technique yields the bound
$\tilde{O}\left(n^{3/4}\sqrt{\bar{L}/\lambda} + n\right)$. That is, we
have the same dependency on the square root of the condition number,
$\sqrt{\bar{L}/\lambda}$, but this term is multiplied by $n^{3/4}$
rather than by $n^{1/2}$. Understanding if this factor can be
eliminated is left to future work.

\paragraph{Related work:} 
In recent years, many randomized methods for optimizing average of
functions have been proposed. For example, SAG~\cite{LSB12-sgdexp},
SVRG~\cite{johnson2013accelerating}, Finito~\cite{defazio2014finito},
SAGA~\cite{defazio2014saga}, S2GD~\cite{konevcny2013semi}, and
UniVr~\cite{allen2015univr}.  All of these methods have similar
convergence rates for strongly convex and smooth problems. Here we
show that SDCA achieves the best known convergence rate for the case in
which individual loss functions are convex, and a slightly worse rate
for the case in which individual loss functions are non-convex. A
systematic study of the convergence rate of the different methods
under non-convex losses is left to future work.

This version of the paper improves upon a previous unpublished version
of the paper \cite{shalev2015sdca} in three aspects. First, the
convergence rate here depends on $\bar{L}$ as opposed to $L_{\max}$ in
\cite{shalev2015sdca}. Second, the version in \cite{shalev2015sdca}
only deals with the regularized case, while here we show that the same
rate can be obtained for unregularized objectives. Last, for the
non-convex case, here we derive the bound
$\tilde{O}\left(n^{3/4}\sqrt{\bar{L}/\lambda} + n\right)$
while in \cite{shalev2015sdca} only the
bound of $\tilde{O}(L_{\max}^2/\lambda^2 + n)$ has been given. 

\cite{csiba2015primal} extended the work of
\cite{shalev2015sdca} to support arbitrary mini-batching schemes, and
\cite{he2015dual} extended the work of \cite{shalev2015sdca} to
support adaptive sampling probabilities. A primal form of SDCA has
been also given in \cite{defazio2014new}. Using SVRG for non-convex
individual functions has been recently studied in
\cite{shamir2014stochastic,jin2015robust}, in the context of fast
computation of the top singular vectors of a matrix.


\section{SDCA without Duality}

We start the section by describing a variant of SDCA that do not rely
on duality. To simplify the presentation, we start in
\secref{sec:regularized} with regularized loss minimization
problems. In \secref{sec:nonRegularized} we tackle the non-regularized
case and in \secref{sec:non-convex} we tackle the non-convex case.

We recall the following basic definitions: A (differentiable) function
$f$ is $\lambda$-strongly convex if for every $u,w$ we have
$f(w)-f(u)\ge \nabla f(u)^\top (w-u) + \frac{\lambda}{2}
\|w-u\|^2$. We say that $f$ is convex if it is $0$-strongly convex. We
say that $f$ is $L$-smooth if $\|\nabla f(w) - \nabla f(u) \| \le L
\|w-u\|$. It is well known that smoothness and convexity also implies
that $f(w)-f(u) \le \nabla f(u)^\top (w-u) + \frac{L}{2} \|w-u\|^2$.

\subsection{Regularized problems} \label{sec:regularized}

In regularized problems, each $f_i$ can be written as $f_i(w) =
\phi_i(w) + \frac{\lambda}{2} \|w\|^2$. 
Similarly to the original SDCA algorithm, we maintain vectors $\alpha_1,\ldots,\alpha_n$, where each
$\alpha_i \in \reals^d$. We call these vectors pseudo-dual
vectors. The algorithm is described below.
\begin{myalgo}{Algorithm 1: Dual-Free SDCA for Regularized Objectives} 
\textbf{Goal:} Minimize $F(w) = \frac{1}{n} \sum_{i=1}^n
\phi_i(w) + \frac{\lambda}{2} \|w\|^2$ \\
\textbf{Input:} Objective $F$, number of iterations $T$,  \+ \\
step size $\eta$,  \\
Smoothness parameters $L_1,\ldots,L_n$ \- \\ 
\textbf{Initialize:} $w^{(0)} = \frac{1}{\lambda\,n} \sum_{i=1}^n 
\alpha_i^{(0)}$ \\ for some $\alpha^{(0)} =
(\alpha_1^{(0)},\ldots,\alpha_n^{(0)})$ \+ \\
$\forall i \in [n]$, $q_i = (L_i + \bar{L})/(2n \bar{L})$ \\ where
$\bar{L} = \frac{1}{n} \sum_{i=1}^n L_i$ \- \\
\textbf{For~} $t=1,\ldots,T$ \+ \\
Pick $i \sim q$, denote $\eta_i =
\frac{\eta}{q_i n}$ \\
Update: \\
 $\alpha_{i}^{(t)} = \alpha_{i}^{(t-1)}  - \eta_i\lambda n
\left( \nabla \phi_i(w^{(t-1)}) + \alpha_{i}^{(t-1)} \right)$ \\
 $w^{(t)} = w^{(t-1)} - \eta_i \left( \nabla \phi_i(w^{(t-1)}) +  \alpha_{i}^{(t-1)} \right)$
\end{myalgo}
Observe that SDCA keeps the primal-dual relation
\[
w^{(t-1)} = \frac{1}{\lambda n} \sum_{i=1}^n  \alpha^{(t-1)}_i
\]
Observe also that the update of $\alpha$ can be rewritten as
\[
\alpha_{i}^{(t)} = (1-\beta_i) \alpha_{i}^{(t-1)}  + \beta_i \left(
  - \nabla \phi_i(w^{(t-1)}) \right) ~,
\]
where $\beta_i = \eta_i \lambda n$.  Namely, the new value of
$\alpha_i$ is a convex combination of its old value and the negative
gradient.  Finally, observe that, conditioned on the value of
$w^{(t-1)}$ and $\alpha^{(t-1)}$, we have that
\begin{align*}
\E_{i \sim q}[w^{(t)} ] &=  w^{(t-1)} - \eta \sum_i \frac{q_i}{q_i n}  \left(  (\nabla \phi_i(w^{(t-1)}) + \alpha_{i}^{(t-1)}) \right) \\
&=  w^{(t-1)}  - \eta \left( \nabla \frac{1}{n}
  \sum_{i=1}^n \phi_i(w^{(t-1)}) + \lambda w^{(t-1)} \right) \\
&= w^{(t-1)} - \eta \nabla P(w^{(t-1)}) ~.
\end{align*}
That is, SDCA is in fact an instance of Stochastic Gradient Descent
(SGD).  As we will see shortly, the advantage of SDCA over a vanilla
SGD algorithm is because the \emph{variance} of the update goes to
zero as we converge to an optimum. 

Our convergence analysis relies on bounding the following potential
function, defined for every $t \ge 0$, 
\begin{equation} \label{eqn:Cdef}
C_t = \frac{\lambda}{2} \|w^{(t)}-w^*\|^2 + \frac{\eta}{n^2} \sum_{i=1}^n [ \frac{1}{q_i} \|\alpha^{(t)}_i  -
\alpha^*_i\|^2] ~,
\end{equation}
where
\begin{equation} \label{eqn:wstar_def}
w^* = \argmin_w F(w),~~\textrm{and}~~ \forall i,~\alpha^*_i = -\nabla \phi_i(w^*) ~.
\end{equation}
Intuitively, $C_t$ measures the distance to the optimum both in primal
and pseudo-dual variables. 
Observe that if $F$ is $L_F$-smooth and convex then
\[
F(w^{(t)})-F(w^*) \le \frac{L_F}{2} \|w^{(t)}-w^*\|^2 \le
\frac{L_F}{\lambda} C_t ~,
\]
and therefore a bound on $C_t$ immediately implies a bound on
the sub-optimality of $w^{(t)}$. 

The following theorem establishes the convergence rate of SDCA for the
case in which each $\phi_i$ is convex.
\begin{theorem} \label{thm:mainReguConv} Assume that each $\phi_i$ is
  $L_i$-smooth and convex, and Algorithm 1 is run with $\eta \le
  \min\left\{ \frac{1}{4 \bar{L}} ~,~ \frac{1}{4\,\lambda n}
  \right\}$.
 Then, for every $t \ge 1$,
\[
\E[C_t] ~\le~ (1-\eta \lambda)^t\, C_0 ~,
\]
where $C_t$ is as defined in \eqref{eqn:Cdef}.  In particular, to
achieve $\E[F(w^{(T)})-F(w^*)] \le \epsilon$ it suffices to set $\eta
= \min\left\{ \frac{1}{4 \bar{L}} ~,~ \frac{1}{4\,\lambda n} \right\}$
and
\[
T \ge 
\tilde{\Omega}\left( \frac{\bar{L}}{\lambda} + n\right) ~.
\]
\end{theorem}

\paragraph{Variance Reduction:} The lemma below tells us that the
variance of the SDCA update decreases as we get closer to the
optimum. 
\begin{lemma} \label{lem:varianceReduction}
Under the same conditions of \thmref{thm:mainReguConv}, the expected
value of $\|w^{(t)}-w^{(t-1)}\|^2$ conditioned on $w^{(t-1)}$ satisfies:
\[
\E[\|w^{(t)}-w^{(t-1)}\|^2] ~\le~ 
3\,\eta \, \left(\tfrac{1}{2} \|w^{(t-1)} - w^*\|^2  + C_{t-1}\right) ~.
\]
\end{lemma}

\subsection{SDCA without regularization} \label{sec:nonRegularized}

We now turn to the case in which the objective is not explicitly
regularized. The algorithm below tackles this problem by a reduction to
the regularized case. In particular, we artificially add
regularization to the objective and compensate for it by adding one
more loss function that cancels out the regularization term. While the
added function is not convex (in fact, it is concave), we prove that
the same convergence rate holds due to the special structure of the
added loss function. 

\begin{myalgo}{Algorithm 2: Dual-Free SDCA for Non-Regularized Objectives} 
\textbf{Goal:} Minimize $F(w) = \frac{1}{n} \sum_{i=1}^n
f_i(w) $ \\
\textbf{Input:} Objective $F$, number of iterations $T$, \+ \\
step size $\eta$, Strong convexity parameter $\lambda$,  \\
Smoothness parameters $L_1,\ldots,L_n$ \- \\ 
\textbf{Define:} \+ \\
For all $i \in [n]$, $\phi_i(w) = \frac{n+1}{n} f_i(w)$, $\tilde{L}_i
= \frac{n+1}{n} L_i$  \\
For $i = n+1$, $\phi_i(w) = \tfrac{-\lambda\,i}{2} \|w\|^2$, $\tilde{L}_i = \lambda\,i$ \- \\
\textbf{Solve:} \+ \\
Rewrite $F$ as $F(w) = \frac{1}{n+1}
\sum_{i=1}^{n+1} \phi_i(w) + \frac{\lambda}{2} \|w\|^2$ \\
Call Algorithm 1 with $F$ above and with $\{\tilde{L}_i\}$
\end{myalgo}

\begin{theorem} \label{thm:mainNonReguConv} Assume that $F$ is
  $\lambda$-strongly convex, that each $f_i$ is
  $L_i$-smooth and convex, and that Algorithm 2 is run with $\eta \le
\min\left\{ \frac{1}{8
    (\bar{L} + \lambda)}  ~,~
  \frac{1}{4\,\lambda (n+1)} \right\}$.
  Then, for every $t \ge 1$,
\[
\E[C_t] ~\le~ (1-\eta \lambda)^t \, C_0 ~,
\]
where $C_t$ is as defined in \eqref{eqn:Cdef}.
In particular, to achieve $\E[F(w^{(T)})-F(w^*)] \le \epsilon$ it
suffices to set $\eta = \min\left\{ \frac{1}{8
    (\bar{L} + \lambda)}  ~,~ \frac{1}{4\,\lambda (n+1)} \right\}$ and
\[
T \ge 
\tilde{\Omega}\left( \frac{\bar{L}}{\lambda} + n\right) ~.
\]
\end{theorem}

\subsection{The non-convex case} \label{sec:non-convex}
We now consider the non-convex case. For simplicity, we focus on the
regularized setting. In the non-regularized setting we can simply
replace every $f_i$ with $\phi_i(w) = f_i(w) -
\frac{\lambda}{2}\|w\|^2$ and apply the regularized setting. Note that
this does not change significantly the smoothness (because $\lambda$
is typically much smaller than the average smoothness of the $f_i$). 

We can apply Algorithm 1 for the non-convex case, and the only change is the choice of
$\eta$, as reflected in the theorem below. 
\begin{theorem} \label{thm:non-convex}
Consider running algorithm 1 on $F$ which is $\lambda$-strongly
convex, assume that each $\phi_i$ is
  $L_i$-smooth, and $\eta \le \min\left\{ \frac{\lambda}{4
    \bar{L}^2}  ~,~
  \frac{1}{4\,\lambda n} \right\}$. 
  Then, for every $t \ge 1$,
\[
\E[C_t] ~\le~ (1-\eta \lambda)^t \, C_0 ~,
\]
where $C_t$ is as defined in \eqref{eqn:Cdef}.
In particular, to achieve $\E[F(w^{(T)})-F(w^*)] \le \epsilon$ it
suffices to set $\eta = \min\left\{ \frac{\lambda}{4
    \bar{L}^2}  ~,~
  \frac{1}{4\,\lambda n} \right\} $ and 
\[
T \ge 
\tilde{\Omega}\left( \frac{
  \bar{L}^2}{\lambda^2} + n\right) ~.
\]
\end{theorem}

As can be seen, the dependence of $T$ on the condition number,
$\frac{\bar{L}}{\lambda}$, is quadratic for the non-convex case, as
opposed to a linear dependency for the convex case. We next show how
to improve the bound using acceleration. 

\subsection{Acceleration}

Accelerated SDCA \cite{ShalevZhangAcc2015} is obtained by solving
(using SDCA) a sequence of problems, where at each iteration, we add
an artificial regularization of the form $\frac{\kappa}{2}
\|w-y^{(t-1)}\|^2$, where $y^{(t-1)}$ is a function of $w^{(t-1)}$ and
$w^{(t-2)}$. The algorithm has been generalized in 
\cite{lin2015universal} to allow the inner solver to be any
algorithm. For completeness, we provide the pseudo-code of the
``Catalyst'' algorithm of \cite{lin2015universal} and its analysis. 

\begin{myalgo}{Algorithm 3: Acceleration}
\textbf{Goal:} Minimize a $\lambda$-strongly convex function $F(w)$ \\
\textbf{Parameters:} $\kappa, T$ \\
\textbf{Initialize:} \+ \\
Initial solution $w^{(0)}$ \\
$\epsilon_0$ s.t. $\epsilon_0 \ge F(w^{(0)})-F(w^*)$ \\
$y^{(0)} = w^{(0)}$, $q = \tfrac{\lambda}{\lambda + \kappa}$\- \\
\textbf{For:} $t=1,\ldots,T$\+ \\
Define $G_t(w) = F(w) + \frac{\kappa}{2} \|w-y^{(t-1)}\|^2$ \\
Set $\epsilon_t = (1-0.9 \sqrt{q}) \, \epsilon_{t-1}$ \\
Find $w^{(t)}$ s.t. $G_t(w^{(t)}) - \min_w G_t(w) \le \epsilon_t$ \\
Set $y^{(t)} = w^{(t)} +
\frac{\sqrt{q}-q}{\sqrt{q}+q} (w^{(t)}-w^{(t-1)}) $
\- \\
\textbf{Output:} $w^{(T)}$
\end{myalgo}

\begin{lemma}  \label{lem:lin}
Fix $\epsilon > 0$ and suppose we run the Acceleration algorithm
(Algorithm 3) for 
\[
T = \Omega\left( \sqrt{\frac{\lambda+\kappa}{\lambda}} \, \log\left(
    \frac{\lambda+\kappa}{\lambda\,\epsilon} \right) \right)
\] iterations. Then, $F(w^{(T)}) - F(w^*) \le \epsilon$. 
\end{lemma}
\begin{proof}
  The lemma follows directly from Theorem 3.1 of
  \cite{lin2015universal} by observing that Algorithm 3 is a
  specification of Algorithm 1 in \cite{lin2015universal} with
  $\alpha_0 = \sqrt{q}$ (which implies that $\alpha_t = \alpha_0$ for
  every $t$), with $\epsilon_t = \epsilon_0 (1-\rho)^t$, and with
  $\rho = 0.9 \sqrt{q}$.
\end{proof}

\begin{theorem} \label{thm:non-convex-acc}
Let $F = \frac{1}{n} \sum_{i=1}^n \phi_i(w) + \frac{\lambda}{2}
\|w\|^2$, assume that each $\phi_i$ is $L_i$ smooth and that $F$ is
$\lambda$-strongly convex. Assume also that $(\bar{L}/\lambda)^2 \ge
3n$ (otherwise we can simply apply $\tilde{O}(n)$ iterations of
Algorithm 1). Then, running Algorithm 3 with parameters $\kappa =
\bar{L}/\sqrt{n}$, $T = \tilde{\Omega}\left(1 + n^{-1/4}
  \sqrt{\bar{L}/\lambda}\right)$, and while at each iteration of
Algorithm 3 using $\tilde{\Omega}\left( n \right)$ iterations of Algorithm 1
to minimize $G_t$, guarantees that $F(w^{(T)})-F(w^*) \le \epsilon$
(with high probability). The total required number of iterations of Algorithm 1
is therefore bounded by
$
\tilde{O}\left( n + n^{3/4}  \sqrt{{\bar{L}}/{\lambda}}  \right) ~.
$
\end{theorem}

Observe that for the case of convex individual functions, accelerating
Algorithm 1 yields the upper bound
$
\tilde{O}\left( n + n^{1/2}  \sqrt{{\bar{L}}/{\lambda}}  \right) ~.
$
Therefore, the convex and non-convex cases have the same dependency on the
condition number, but the non-convex case has a worse dependence on $n$.




\section{Proofs}

\subsection{Proof of \thmref{thm:mainReguConv}}

Observe that $0 = \nabla F(w^*) = \frac{1}{n} \sum_i \nabla
\phi_i(w^*) + \lambda w^*$, which implies that $w^* = \frac{1}{\lambda n
  } \sum_i \alpha_i^*$, where $\alpha_i^* = - \nabla \phi_i(w^*)$.

Define $u_i =- \nabla \phi_i(w^{(t-1)})$ and $v_i = -u_i +
\alpha_i^{(t-1)}$. We also denote two potentials:
\[
A_t = \sum_{j=1}^n \frac{1}{q_j} \|\alpha^{(t)}_j  - \alpha^*_j\|^2
~~~,~~~
B_t = \|w^{(t)}-w^*\|^2 ~.
\]
We will first analyze the evolution of $A_t$ and $B_t$.  If on round
$t$ we update using element $i$ then $\alpha_i^{(t)} = (1-\beta_i)
\alpha_i^{(t-1)} + \beta_i u_i$. It
follows that,
\begin{align}  \label{eqn:Aevo}
&A_{t-1}- A_t =  - \frac{1}{q_i} \|\alpha^{(t)}_i  - \alpha^*_i\|^2 +  \frac{1}{q_i}\|\alpha^{(t-1)}_i  -
\alpha^*_i\|^2  \\ \nonumber
&= - \frac{1}{q_i}\|(1-\beta_i) (\alpha^{(t-1)}_i  - \alpha^*_i) +
\beta_i(u_i - \alpha_i^*) \|^2  \\ \nonumber
& ~~~~+ \frac{1}{q_i}\|\alpha^{(t-1)}_i  -
\alpha^*_i\|^2 \\ \nonumber
&=  \frac{1}{q_i} ( -(1-\beta_i)\|\alpha^{(t-1)}_i  - \alpha^*_i\|^2 -
  \beta_i \|u_i - \alpha_i^*\|^2 \\ \nonumber 
&~~~+ \beta_i(1-\beta_i)\|\alpha^{(t-1)}_i - u_i \|^2 + \|\alpha^{(t-1)}_i  -
\alpha^*_i\|^2 ~) \\ \nonumber
&= \frac{\beta_i}{q_i} \left( \|\alpha^{(t-1)}_i  - \alpha^*_i\|^2 -
  \|u_i - \alpha_i^*\|^2 +
  (1-\beta_i) \|v_i\|^2 \right) \\
&= \frac{\eta\,\lambda}{q_i^2} \left( \|\alpha^{(t-1)}_i  - \alpha^*_i\|^2 -
  \|u_i - \alpha_i^*\|^2 +
  (1-\beta_i) \|v_i\|^2 \right) ~.
\end{align}
Taking expectation w.r.t. $i \sim q$ we obtain
\begin{align} \nonumber
\E[&A_{t-1}- A_t ] = \\
&\eta \lambda \sum_{i=1}^n \frac{1}{q_i}  \left( \|\alpha^{(t-1)}_i  - \alpha^*_i\|^2 -
  \|u_i - \alpha_i^*\|^2 +
  (1-\beta_i) \|v_i\|^2 \right) \\
&= \eta \lambda \left( A_{t-1} + \sum_{i=1}^n \frac{1}{q_i}  \left( -
  \|u_i - \alpha_i^*\|^2 +
  (1-\beta_i) \|v_i\|^2 \right) \right) ~. \label{eqn:EApot}
\end{align}
As to the second potential, we have
\begin{align} \label{eqn:Bevo}
B_{t-1}-B_t &= 
-\|w^{(t)}-w^*\|^2+\|w^{(t-1)}-w^*\|^2 \\ \nonumber 
&= 2 \,  (w^{(t-1)}-w^*)^\top (\eta\, v_i) -  \eta_i^2 \|v_i\|^2 ~.
\end{align}
Taking expectation w.r.t. $i \sim q$ and noting that $\E_{i \sim q}
(\eta_i v_i) = \eta \nabla F(w^{(t-1)})$ we obtain
\begin{align} \label{eqn:EBpot}
\E[B_{t-1}-B_t] = & 2 \eta\, (w^{(t-1)}-w^*)^\top \nabla F(w^{(t-1)})
\\ \nonumber &-
\frac{\eta^2}{n^2} \sum_i \frac{1}{q_i} \|v_i\|^2 ~.
\end{align}

We now take a potential of the form $C_t = c_a A_t + c_b B_t$. 
Combining \eqref{eqn:EApot} and \eqref{eqn:EBpot} we obtain
\begin{align} \nonumber
\E[C_{t-1}-C_t] &= c_a \eta \lambda A_{t-1}  - c_a \eta \lambda \sum_i
\frac{1}{q_i} \|u_i - \alpha^*_i\|^2 \\ \nonumber &+ 2 c_b \eta (w^{(t-1)}-w^*)^\top
\nabla F(w^{(t-1)}) \\
&+ \sum_i \frac{1}{q_i} \|v_i\|^2 \left( c_a \eta \lambda (1-\beta_i) -
  \frac{c_b \eta^2}{n^2} \right) \label{eqn:viCoeff}
\end{align}
We will choose the parameters $\eta, c_a, c_b$ such that
\begin{align} \label{eqn:etaConditions}
&\eta \le \min\left\{\frac{q_i}{2 \lambda} ~,~ \frac{1}{4 \bar{L}}
\right\} ~~\textrm{and}~~ 
\frac{c_b}{c_a} = \frac{\lambda n^2}{2 \eta}
\end{align}
This implies that $\beta_i = \eta_i \lambda n = \frac{\eta
  \lambda}{q_i} \le 1/2$, and therefore the term in \eqref{eqn:viCoeff} is non-negative. 
Next, due to strong convexity of $F$ we have that
\begin{align*}
&(w^{(t-1)}-w^*)^\top \nabla F(w^{(t-1)}) \\
~&\ge~ F(w^{(t-1)})-F(w^*) +
\frac{\lambda}{2} \| w^{(t-1)}-w^*\|^2  ~.
\end{align*}
Therefore,
\begin{align} \nonumber
&\E[C_{t-1}-C_t] = c_a \eta \lambda A_{t-1}  - c_a \eta \lambda \sum_i
\frac{1}{q_i} \|u_i - \alpha^*_i\|^2 \\ \nonumber & + 2 c_b \eta
(F(w^{(t-1)})-F(w^*)) + c_b \eta \lambda B_{t-1}  \\ \nonumber
&= \eta\,\lambda\,C_{t-1} + \\ & \eta \left(2 c_b (F(w^{(t-1)})-F(w^*)) -
  c_a \lambda \sum_i
\frac{1}{q_i} \|u_i - \alpha^*_i\|^2 \right) . \label{eqn:EC1}
\end{align}
Note that $u_i - \alpha^*_i = \nabla \phi_i(w^{(t-1)}) -
\nabla \phi_i(w^*)$. In \lemref{lem:smoothConv} we show that when
$\phi_i$ is $L_i$ smooth and convex then 
\begin{align} \label{eqn:lem:smoothConv}
&\|\nabla \phi_i(w^{(t-1)}) -
\nabla \phi_i(w^*)\|^2 \\ \nonumber &\le 2 \, L_i\, (\phi_i(w^{(t-1)}) - \phi_i(w^*)
- \nabla \phi_i(w^*)^\top (w^{(t-1)}-w^*))
\end{align}
Therefore, denoting $\tau = \left(2\, \max_i \frac{L_i}{q_i}\right)$
we obtain that
\begin{align} \label{eqn:departPoint}
&\sum_i \frac{1}{q_i} \|u_i -
\alpha^*_i\|^2 = \sum_i \frac{1}{q_i} \|\nabla \phi_i(w^{(t-1)}) -
\nabla \phi_i(w^*)\|^2 \\ \nonumber
&\le \tau \, \sum_i  (\phi_i(w^{(t-1)}) - \phi_i(w^*)
- \nabla \phi_i(w^*)^\top (w^{(t-1)}-w^*)) \\ \nonumber
&= \tau \, n\,\left(
  F(w^{(t-1)})-F(w^*)-\frac{\lambda}{2} \|w^{(t-1)}-w^*\|^2 \right) \\ 
&\le \tau\, n\,\left(
  F(w^{(t-1)})-F(w^*) \right) ~.
\end{align}
The definition of $q_i$ implies that for every $i$,
\begin{equation} \label{eqn:Li_over_q}
\frac{L_i}{q_i} = 2n\bar{L}\,\frac{L_i}{L_i + \bar{L}} \le 2n\bar{L} ~.
\end{equation}
Combining this with \eqref{eqn:departPoint} and \eqref{eqn:EC1} we obtain
\begin{align*}
&\E[C_{t-1}-C_t]  \ge \\
&\eta\,\lambda\,C_{t-1} + \eta \left(2 c_b - 4
  n^2\bar{L} \lambda c_a \right)
(F(w^{(t-1)})-F(w^*))
\end{align*}
Plugging the value of $c_b = \frac{c_a \lambda n^2}{2\eta}$ yields
that the coefficient in the last term is
\[
2 \frac{c_a \lambda n^2}{2\eta}- 4
  n^2\bar{L} \lambda c_a = c_a \lambda n^2 \left( \frac{1}{\eta} - 4
  \bar{L}  \right) \ge 0 ~,
\]
where we used the choice of $\eta \le \frac{1}{4 \bar{L}}$. In
summary, we have shown that $\E[C_{t-1}-C_t]  \ge
\eta\,\lambda\,C_{t-1}$, which implies that
\[
\E[C_t] ~\le~ (1-\eta\,\lambda) \,C_{t-1} ~.
\]
Taking expectation over $C_{t-1}$ and continue recursively, we
obtain that $\E[C_t] ~\le~ (1-\eta\,\lambda)^t \, C_0 ~\le~ e^{-\eta\,\lambda\,t} \, C_0$. 

Finally, since $q_i \ge 1/(2n)$ for every $i$, we can choose 
\[
\eta = \min\left\{ \frac{1}{4
    \bar{L}}  ~,~
  \frac{1}{4\,\lambda n} \right\}
\]
and therefore 
\[
\frac{1}{\eta \lambda} \le 4 \left(n + \frac{
  \bar{L}}{\lambda} \right) ~.
\]
The proof is concluded by choosing $c_b=\lambda/2$ and
$c_a = \eta/n^2$. 

\subsection{Proof of \lemref{lem:varianceReduction}}
We have:
\begin{align*}
&\E[\|w^{(t)}-w^{(t-1)}\|^2] = \sum_i q_i \eta_i^2 \|\nabla
\phi_i(w^{(t-1)})+ \alpha_i^{(t-1)}\|^2 \\
&\le \frac{3 \eta^2}{n^2} \sum_i \frac{1}{q_i} (\|\nabla
\phi_i(w^{(t-1)}) + \alpha_i^*\|^2 \\&\hspace{2cm} + \|\alpha_i^{(t-1)} - \alpha_i^*
\|^2) \\ &\hspace{5cm}  \textrm{(triangle inequality)}\\
&= \frac{3 \eta^2}{n^2} \sum_i (\tfrac{1}{q_i} \|\nabla
  \phi_i(w^{(t-1)}) - \nabla \phi_i(w^*)\|^2 \\ &\hspace{4cm} + \tfrac{1}{q_i}\|\alpha_i^{(t-1)} - \alpha_i^*
\|^2 ) \\
&\le \frac{3 \eta^2}{n^2} \sum_i \left(2n\bar{L} \,\|w^{(t-1)} - w^*\|^2 + \tfrac{1}{q_i}\|\alpha_i^{(t-1)} - \alpha_i^*
\|^2 \right) \\ &\hspace{5cm} \textrm{(smoothness and \eqref{eqn:Li_over_q})}\\
&\le 3\,\eta \, \left(\tfrac{1}{2} \|w^{(t-1)} - w^*\|^2  +
  C_{t-1}\right)  \\ &\hspace{5cm} 
(\textrm{because } \eta \le \tfrac{1}{4\bar{L}} )~.
\end{align*}

\subsection{Proof of \thmref{thm:mainNonReguConv}}
The beginning of the proof is identical to the proof of
\thmref{thm:mainReguConv}. The change starts in
\eqref{eqn:departPoint}, where we cannot apply
\eqref{eqn:lem:smoothConv} to $\phi_{n+1}$ because it is not
convex. To overcome this, we first apply \eqref{eqn:lem:smoothConv} to
$\phi_1,\ldots,\phi_n$, and obtain that
\begin{align*} 
&\sum_{i=1}^n \frac{1}{q_i} \|u_i -
\alpha^*_i\|^2 = \sum_{i=1}^n \frac{1}{q_i} \|\nabla \phi_i(w^{(t-1)}) -
\nabla \phi_i(w^*)\|^2 \\ 
&\le \left(2\, \max_i \frac{\tilde{L}_i}{q_i}\right) \,\cdot\\
& \sum_{i=1}^n  (\phi_i(w^{(t-1)}) - \phi_i(w^*)
- \nabla \phi_i(w^*)^\top (w^{(t-1)}-w^*)) \\
&= 2\,  (n+1) \, \left(\max_i \frac{\tilde{L}_i}{q_i}\right) \, (F(w^{(t-1)}) -
F(w^*)) ~,
\end{align*}
where the last equality follows from the fact that $\sum_{i=1}^n
\phi_i(w) = (n+1) F(w)$, which also implies that $\sum_i
\nabla \phi_i(w^*) = 0$. 
In addition, since $\phi_{n+1}(w) =
-\frac{\lambda(n+1)}{2} \|w\|^2$, we have 
\begin{align*}
&\frac{1}{q_{n+1}} \|\nabla \phi_{n+1}(w)-\nabla \phi_{n+1}(w^*)\|^2
\\ &= \frac{\lambda^2 (n+1)^2}{q_{n+1}}
\|w-w^*\|^2 \\
&= 2\,(n+1)\,\frac{\tilde{L}_{n+1}}{q_{n+1}} \, \cdot \,\frac{\lambda}{2} \,\|w-w^*\|^2 \\
&\le 2\,(n+1)\,\frac{\tilde{L}_{n+1}}{q_{n+1}}  (F(w)-F(w^*)) ~,
\end{align*}
where the last inequality is because of the $\lambda$-strong convexity
of $F$. Combining the two inequalities, we obtain an analogue of
\eqref{eqn:departPoint},
\begin{align*}
&\sum_{i=1}^{n+1} \frac{1}{q_i} \|u_i-\alpha_i^*\|^2 \\
&\le~ 
4\,(n+1)\, \left(\max_{i \in [n+1]} \frac{\tilde{L}_i}{q_i}\right) \, (F(w^{(t-1)}) -
F(w^*)) ~.
\end{align*}
The rest of the proof is almost identical, except that we have $n$
replaced by $n+1$ and $\bar{L}$ replaced by $\tilde{L} := \frac{1}{n+1}
\sum_{i=1}^n \tilde{L}_i$. We now need to choose
\[
\eta = \min\left\{ \frac{1}{8
    \tilde{L}}  ~,~
  \frac{1}{4\,\lambda (n+1)} \right\} ~.
\]
Observe that, 
\[
(n+1)\tilde{L} = \frac{n+1}{n} \left(\sum_{i=1}^n L_i \right) + \lambda (n+1) = (n+1)
(\bar{L} + \lambda) ~,
\]
so we can rewrite 
\[
\eta = \min\left\{ \frac{1}{8
    (\bar{L} + \lambda)}  ~,~
  \frac{1}{4\,\lambda (n+1)} \right\} ~.
\]
This yields
\[
\frac{1}{\eta \lambda} \le  4 \left(n + 3 + \frac{2
  \bar{L}}{\lambda} \right) ~.
\]

\subsection{Proof of \thmref{thm:non-convex}}
The beginning of the proof is identical to the proof of
\thmref{thm:mainReguConv} up to \eqref{eqn:viCoeff}.

We will choose the parameters $\eta, c_a, c_b$ such that
\begin{align} \label{eqn:etaConditions:non-convex}
&\eta \le \min\left\{\frac{q_i}{2 \lambda} ~,~ \frac{1}{4 \bar{L}}
\right\} ~~\textrm{and}~~ 
\frac{c_b}{c_a} = \frac{\lambda n^2}{2 \eta}
\end{align}
This implies that $\beta_i = \eta_i \lambda n = \frac{\eta
  \lambda}{q_i} \le 1/2$, and therefore the term in \eqref{eqn:viCoeff} is non-negative. 
Next, due to strong convexity of $F$ we have that
\begin{align*}
&(w^{(t-1)}-w^*)^\top \nabla F(w^{(t-1)})  \\
~&\ge~ F(w^{(t-1)})-F(w^*) +
\frac{\lambda}{2} \| w^{(t-1)}-w^*\|^2  
\\ &\ge~ \lambda \|w^{(t-1)} - w^*\|^2 ~.
\end{align*}
Therefore,
\begin{align} \nonumber
&\E[C_{t-1}-C_t] \\ \nonumber &= c_a \eta \lambda A_{t-1}  - c_a \eta \lambda \sum_i
\frac{1}{q_i} \|u_i - \alpha^*_i\|^2 + 2 c_b \eta \lambda B_{t-1}  \\
&= \eta\,\lambda\,C_{t-1} + \eta\,\lambda \left(c_b B_{t-1} -
  c_a \sum_i
\frac{1}{q_i} \|u_i - \alpha^*_i\|^2 \right) ~. \label{eqn:EC1:non-convex}
\end{align}
Next, we use the smoothness of the $\phi_i$ to get
\begin{align*} 
&\sum_i \frac{1}{q_i} \|u_i -
\alpha^*_i\|^2 = \sum_i \frac{1}{q_i} \|\nabla \phi_i(w^{(t-1)}) -
\nabla \phi_i(w^*)\|^2 \\ 
&\le \sum_i \frac{L_i^2}{q_i} \|w^{(t-1)} -w^*\|^2 = B_{t-1} \sum_i \frac{L_i^2}{q_i} ~.
\end{align*}
The definition of $q_i$ implies that for every $i$,
\[
\frac{L_i}{q_i} = 2n\bar{L}\,\frac{L_i}{L_i + \bar{L}} \le 2n\bar{L} ~,
\]
so by combining with \eqref{eqn:EC1:non-convex} we obtain
\[
\E[C_{t-1}-C_t]  \ge \eta\,\lambda\,C_{t-1} + \eta \lambda \left(c_b -
  2 n^2 \bar{L}^2 c_a \right) B_{t-1}
\]
The last term will be non-negative if $\frac{c_b}{c_a} \ge 2 n^2
\bar{L}^2$. Since we chose $\frac{c_b}{c_a} = \frac{\lambda n^2}{2
  \eta}$ we obtain the requirement
\[
 \frac{\lambda n^2}{2
  \eta} \ge 2 n^2 \bar{L}^2~\Rightarrow~ \eta \le \frac{\lambda}{4 \bar{L}^2} ~.
\]
In summary, we have shown that $\E[C_{t-1}-C_t]  \ge
\eta\,\lambda\,C_{t-1}$. The rest of the proof is identical, but the
requirement on $\eta$ is
\[
\eta \le \min\left\{ \frac{\lambda}{4
    \bar{L}^2}  ~,~
  \frac{1}{4\,\lambda n} \right\} ~,
\]
and therefore 
\[
\frac{1}{\eta \lambda} \le 4 \left(n + \frac{
  \bar{L}^2}{\lambda^2} \right) ~.
\]

\section{Proof of \thmref{thm:non-convex-acc}}
\begin{proof}
Each iteration of Algorithm 3 requires to minimize $G_t$ to accuracy $\epsilon_t \le O(1)\,
(1-\rho)^t$, where $\rho = 0.9\,\sqrt{q}$. If $t \le T$ where $T$ is
as defined in \lemref{lem:lin}, then we have that, 
\[
-t \log(1-\rho) \le - T \log(1-\rho) =
\frac{-\log(1-\rho)}{\rho}  \log\left( \frac{800}{q\,\epsilon} \right)
\]
Using \lemref{lem:technicallin}, $\frac{-\log(1-\rho)}{\rho} \le 2$
for every $\rho \in (0,1/2)$. In our case, $\rho$ is indeed in
$(0,1/2)$ because of the definition of $\kappa$ and our
assumption that $(\bar{L}/\lambda)^2 \ge
3n$. Hence, 
\[
\log(\tfrac{1}{\epsilon_t}) ~=~ O(\log((\lambda + \kappa)/(\lambda
\epsilon))) ~.
\]
Combining this with \thmref{thm:non-convex}, and using the definition
of $G_t$, we obtain that the number of iterations
required\footnote{While \thmref{thm:non-convex} bounds the expected
  sub-optimality, by techniques similar to \cite{ShalevZhangAcc2015} it can be converted
  to a bound that holds with high probability.} by each
application of Algorithm 3 is
\[
\tilde{O}\left(  \frac{
  (\bar{L} + \kappa)^2}{(\lambda+\kappa)^2} + n \right) = \tilde{O}(n)
~,
\]
where in the equality we used the definition of $\kappa$. 
Finally, multiplying this by the value of $T$ as given in \lemref{lem:lin} we obtain 
(ignoring log-terms):
\[
\sqrt{1 + \frac{\kappa}{\lambda}} \, n ~\le~ 
(1 + \sqrt{\frac{\kappa}{\lambda}})\, n  = 
 n +
n^{3/4} \sqrt{\frac{\bar{L}}{\lambda}} ~.
\]
\end{proof}

\subsection{Technical Lemmas}

\begin{lemma} \label{lem:smoothConv}
Assume that $\phi$ is $L$-smooth and convex. Then, for every
$w$ and $u$, 
\[ 
\|\nabla \phi(w) - \nabla \phi(u)\|^2 \le 
 2L\left[ \phi(w) - \phi(u) - \nabla \phi(u)^\top (w-u)\right] ~.
\]
\end{lemma}
\begin{proof}
For every $i$, define
\[
g(w) = \phi(w) - \phi(u) - \nabla \phi(u)^\top (w-u) ~.
\]
Clearly, since $\phi$ is $L$-smooth so is $g$. In addition, by
convexity of $\phi$ we have $g(w) \ge 0$ for all $w$. It follows
that $g$ is non-negative and smooth, and therefore, it is
self-bounded (see Section 12.1.3 in \cite{MLbook}):
\[
\|\nabla g(w)\|^2 \le 2L g(w) ~.
\]
Using the definition of $g$, we obtain
\begin{align*}
& \|\nabla \phi(w) - \nabla \phi(u)\|^2 \\
&= \|\nabla g(w)\|^2 \le 2L g(w) \\
&= 2L\left[ \phi(w) - \phi(u) - \nabla \phi(u)^\top (w-u)\right] ~.
\end{align*}
\end{proof}

\begin{lemma} \label{lem:technicallin}
For $a \in (0,1/2)$ we have $-\log(1-a)/a \le 1.4$. 
\end{lemma}
\begin{proof}
Denote $g(a) = -\log(1-a)/a$. It is easy to verify that the derivative
of $g$ in $(0,1/2)$ is positive and that $g(0.5) \le 1.4$. The proof follows.
\end{proof}

\section{Summary}
We have described and analyzed a dual free version of SDCA that
supports non-regularized objectives and non-convex individual loss
functions. Our analysis shows a linear rate of convergence for all of
these cases. Two immediate open questions are whether the
worse dependence on the condition number for the non-accelerated
result for the non-convex case is necessary, and whether the
factor $n^{3/4}$ in \thmref{thm:non-convex-acc} can be reduced to
$n^{1/2}$. 

\paragraph{Acknowledgements:} In a previous draft of this paper, the
bound for the non-convex case was $n^{5/4} + n^{3/4}
\sqrt{\bar{L}/\lambda}$. We thank Ohad Shamir for showing us how to
derive the improved bound of $n + n^{3/4} \sqrt{\bar{L}/\lambda}$.
The work is supported by ICRI-CI and by the European Research
Council (TheoryDL project).

{
\bibliography{curRefs}

\begin{thebibliography}{19}
\providecommand{\natexlab}[1]{#1}
\providecommand{\url}[1]{\texttt{#1}}
\expandafter\ifx\csname urlstyle\endcsname\relax
  \providecommand{\doi}[1]{doi: #1}\else
  \providecommand{\doi}{doi: \begingroup \urlstyle{rm}\Url}\fi

\bibitem[Agarwal \& Bottou(2014)Agarwal and Bottou]{agarwal2014lower}
Agarwal, Alekh and Bottou, Leon.
\newblock A lower bound for the optimization of finite sums.
\newblock In \emph{ICML}, 2014.

\bibitem[Allen-Zhu \& Yuan(2015)Allen-Zhu and Yuan]{allen2015univr}
Allen-Zhu, Zeyuan and Yuan, Yang.
\newblock Univr: A universal variance reduction framework for proximal
  stochastic gradient method.
\newblock \emph{arXiv preprint arXiv:1506.01972}, 2015.

\bibitem[Arjevani et~al.(2015)Arjevani, Shalev-Shwartz, and
  Shamir]{arjevani2015lower}
Arjevani, Yossi, Shalev-Shwartz, Shai, and Shamir, Ohad.
\newblock On lower and upper bounds for smooth and strongly convex optimization
  problems.
\newblock \emph{arXiv preprint arXiv:1503.06833}, 2015.

\bibitem[Csiba \& Richt{\'a}rik(2015)Csiba and Richt{\'a}rik]{csiba2015primal}
Csiba, Dominik and Richt{\'a}rik, Peter.
\newblock Primal method for erm with flexible mini-batching schemes and
  non-convex losses.
\newblock \emph{arXiv preprint arXiv:1506.02227}, 2015.

\bibitem[Defazio(2014)]{defazio2014new}
Defazio, Aaron.
\newblock \emph{New Optimisation Methods for Machine Learning}.
\newblock PhD thesis, Australian National Univer- sity, 2014.

\bibitem[Defazio et~al.(2014{\natexlab{a}})Defazio, Bach, and
  Lacoste-Julien]{defazio2014saga}
Defazio, Aaron, Bach, Francis, and Lacoste-Julien, Simon.
\newblock Saga: A fast incremental gradient method with support for
  non-strongly convex composite objectives.
\newblock In \emph{Advances in Neural Information Processing Systems}, pp.\
  1646--1654, 2014{\natexlab{a}}.

\bibitem[Defazio et~al.(2014{\natexlab{b}})Defazio, Caetano, and
  Domke]{defazio2014finito}
Defazio, Aaron~J, Caetano, Tib{\'e}rio~S, and Domke, Justin.
\newblock Finito: A faster, permutable incremental gradient method for big data
  problems.
\newblock \emph{arXiv preprint arXiv:1407.2710}, 2014{\natexlab{b}}.

\bibitem[He \& Tak{\'a}{\v{c}}(2015)He and Tak{\'a}{\v{c}}]{he2015dual}
He, Xi and Tak{\'a}{\v{c}}, Martin.
\newblock Dual free sdca for empirical risk minimization with adaptive
  probabilities.
\newblock \emph{arXiv preprint arXiv:1510.06684}, 2015.

\bibitem[Jin et~al.(2015)Jin, Kakade, Musco, Netrapalli, and
  Sidford]{jin2015robust}
Jin, Chi, Kakade, Sham~M, Musco, Cameron, Netrapalli, Praneeth, and Sidford,
  Aaron.
\newblock Robust shift-and-invert preconditioning: Faster and more sample
  efficient algorithms for eigenvector computation.
\newblock \emph{arXiv preprint arXiv:1510.08896}, 2015.

\bibitem[Johnson \& Zhang(2013)Johnson and Zhang]{johnson2013accelerating}
Johnson, Rie and Zhang, Tong.
\newblock Accelerating stochastic gradient descent using predictive variance
  reduction.
\newblock In \emph{Advances in Neural Information Processing Systems}, pp.\
  315--323, 2013.

\bibitem[Kone{\v{c}}n{\`y} \& Richt{\'a}rik(2013)Kone{\v{c}}n{\`y} and
  Richt{\'a}rik]{konevcny2013semi}
Kone{\v{c}}n{\`y}, Jakub and Richt{\'a}rik, Peter.
\newblock Semi-stochastic gradient descent methods.
\newblock \emph{arXiv preprint arXiv:1312.1666}, 2013.

\bibitem[{Le Roux} et~al.(2012){Le Roux}, {Schmidt}, and {Bach}]{LSB12-sgdexp}
{Le Roux}, Nicolas, {Schmidt}, Mark, and {Bach}, Francis.
\newblock A stochastic gradient method with an exponential convergence rate for
  finite training sets.
\newblock In \emph{Advances in Neural Information Processing Systems}, pp.\
  2663--2671, 2012.

\bibitem[Lin et~al.(2015)Lin, Mairal, and Harchaoui]{lin2015universal}
Lin, Hongzhou, Mairal, Julien, and Harchaoui, Zaid.
\newblock A universal catalyst for first-order optimization.
\newblock In \emph{Advances in Neural Information Processing Systems}, pp.\
  3366--3374, 2015.

\bibitem[Shalev-Shwartz \& Zhang(2015)Shalev-Shwartz and
  Zhang]{ShalevZhangAcc2015}
Shalev-Shwartz, S. and Zhang, T.
\newblock Accelerated proximal stochastic dual coordinate ascent for
  regularized loss minimization.
\newblock \emph{Mathematical Programming SERIES A and B (to appear)}, 2015.

\bibitem[Shalev-Shwartz(2015)]{shalev2015sdca}
Shalev-Shwartz, Shai.
\newblock Sdca without duality.
\newblock \emph{arXiv preprint arXiv:1502.06177}, 2015.

\bibitem[Shalev-Shwartz \& Ben-David(2014)Shalev-Shwartz and Ben-David]{MLbook}
Shalev-Shwartz, Shai and Ben-David, Shai.
\newblock \emph{Understanding Machine Learning: From Theory to Algorithms}.
\newblock Cambridge university press, 2014.

\bibitem[Shalev-Shwartz \& Zhang(2013)Shalev-Shwartz and Zhang]{ShalevZh2013}
Shalev-Shwartz, Shai and Zhang, Tong.
\newblock Stochastic dual coordinate ascent methods for regularized loss
  minimization.
\newblock \emph{Journal of Machine Learning Research}, 14:\penalty0 567--599,
  Feb 2013.

\bibitem[Shamir(2015)]{shamir2014stochastic}
Shamir, Ohad.
\newblock A stochastic pca and svd algorithm with an exponential convergence
  rate.
\newblock In \emph{ICML}, 2015.

\bibitem[Xiao \& Zhang(2014)Xiao and Zhang]{xiao2014proximal}
Xiao, Lin and Zhang, Tong.
\newblock A proximal stochastic gradient method with progressive variance
  reduction.
\newblock \emph{SIAM Journal on Optimization}, 24\penalty0 (4):\penalty0
  2057--2075, 2014.

\end{thebibliography}
\bibliographystyle{icml2016}
}

\end{document}